\documentclass{article}
\usepackage{ijcai18}

\usepackage{times}
\usepackage{xcolor}
\usepackage{soul}
\usepackage[utf8]{inputenc}
\usepackage[small]{caption}

\usepackage{hyperref}
\usepackage[pdftex]{graphicx} 
\usepackage{bm}
\usepackage{amsfonts}
\usepackage{amsmath}
\usepackage{amsthm}
\usepackage{subfigure} 
\usepackage{algorithm}
\usepackage{algorithmic}
\newtheorem{proposition}{Proposition}
\newtheorem{theorem}{Theorem}
\newtheorem{lemma}{Lemma}
\newtheorem{corollary}{Corollary}
\newcommand{\ext}{\mathop{\mathrm{ext}}}
\newcommand{\argmax}{\mathop{\mathrm{argmax}}}
\newcommand{\argext}{\mathop{\mathrm{argext}}}

\title{Balancing Two-Player Stochastic Games with Soft Q-Learning}

\author{
Jordi Grau-Moya, 
Felix Leibfried and
Haitham Bou-Ammar
\\ 
PROWLER.io\\
jordi@prowler.io,
felix@prowler.io,
haitham@prowler.io
}

\begin{document}

\maketitle

\begin{abstract}
Within the context of video games the notion of perfectly rational agents can be undesirable as it leads to uninteresting situations, where humans face tough adversarial decision makers. Current frameworks for stochastic games and reinforcement learning prohibit tuneable strategies as they seek optimal performance. In this paper, we enable such tuneable behaviour by generalising soft Q-learning to stochastic games, where more than one agent interact strategically. We contribute both theoretically and empirically. On the theory side, we show that games with soft Q-learning exhibit a unique value and generalise team games and zero-sum games far beyond these two extremes to cover a continuous spectrum of gaming behaviour. Experimentally, we show how tuning agents' constraints affect performance and demonstrate, through a neural network architecture, how to reliably balance games with high-dimensional representations. 
\end{abstract}

\section{Introduction}

Stochastic Games (SG) provide a natural extension of reinforcement learning~\cite{sutton1998reinforcement,mnih2015human,busoniu2010reinforcement,peters2010relative} to multiple agents, where adapting  strategies in presence of humans or other agents is necessary~\cite{shapley1953stochastic,littman1994markov,littman2001friend}. 

Current frameworks for stochastic games assume perfectly rational agents -- an assumption that is violated in a variety of real-world scenarios, e.g., human-robot interaction~\cite{goodrich2007human}, and pick-up and drop-off domains \cite{agussurja2012toward}. In the context of computer games, the main focus of this paper, such a problem of perfect rationality is even more amplified. Here, in fact, it is not desirable to design agents that seek optimal behaviour as this leads humans to become quickly uninterested when playing against adversarial agents that are impossible to defeat~\cite{hunicke2005case}. Hence, to design games with adaptable and balancing properties tailored to human-level performance, there is a need to extend state-of-the-art SG beyond optimality to allow for tuneable behaviour.

One method to produce tuneable behaviour in reinforcement learning is to introduce an adjustable Kullback-Leibler (KL) constraint between the agent's policy and a reference one. In particular, by increasingly strengthening this constraint we can obtain policies increasingly close to the reference policy, and vice versa. Bounding policy updates in such a manner has been previously introduced in literature under different names. Examples include KL control, relative entropy policy search~\cite{peters2010relative}, path integral control~\cite{kappen2005path,braun2011path}, information-theoretic bounded rationality~\cite{ortega2013thermodynamics}, information-theory of decisions and actions~\cite{tishby2011information,rubin2012trading}, and soft Q-learning~\cite{fox2016taming,haarnoja2017reinforcement}. Targeted problems using these methods are also wide-spread, e.g., tackling the overestimation problem in tabular Q-learning~\cite{fox2016taming} and in Deep Q-networks~\cite{leibfried2017information}, accounting for model misspecification~\cite{grau2016planning}, introducing safe policy updates in robot learning~\cite{schulman2015trust}, and inducing risk-sensitive control~\cite{vandenBroek:2010}.

\textbf{Contributions:} Though abundant in literature, previous works only consider single-agent problems and are not readily applicable to stochastic games, which consider more than one interacting entity. With game balancing as our motivation, we propose a novel formulation of SG where agents are subject to KL constraints. In particular, our formulation introduces two KL constraints, one for each agent, limiting the space of available policies, which, in turn, enables tuneable behaviour. We then introduce an online strategy that can be used for game-play balancing even in high-dimensional spaces through a neural network architecture. 

In short, the contributions of this paper can be summarised as: (1) proving convergence of the two-player soft Q-learning to a fixed point through contractions; (2) generalising team and zero-sum games in a continuous fashion and showing a unique value; (3) demonstrating convergence to correct behaviour by tuning the KL constraints on a simplified grid-world scenario; (4) extending our method to handle high-dimensional spaces; and (5) inferring opponent's Lagrange multiplier by maximum-likelihood, and demonstrating game-balancing behaviour on the game of Pong.

\section{Background}\label{sec:Background}
\subsection{Reinforcement Learning}
In reinforcement learning (RL)~\cite{sutton1998reinforcement} an agent interacts with an unknown environment to determine an optimal policy that maximises total expected return. These problems are formalised as Markov decision processes (MDPs). Formally, an MDP is defined as the tuple $\left\langle \mathcal{S}, \mathcal{A}, \mathcal{T}, \mathcal{R}, \gamma\right\rangle$ where $\mathcal{S}$ is the state space, $\mathcal{A}$ the action space, and $\mathcal{T}: \mathcal{S}\times \mathcal{A} \times \mathcal{S} \rightarrow [0,1]$ denotes the state transition density. Namely, when being in state $\bm{s}_{t} \in \mathcal{S}$ and applying an action $\bm{a}_{t}\in \mathcal{A}$, the agent transitions to $\bm{s}_{t+1} \sim \mathcal{T}(\bm{s}_{t+1}|\bm{s}_{t},\bm{a}_{t})$. The reward function $\mathcal{R}:\mathcal{S} \times \mathcal{A} \rightarrow \mathbb{R}$ quantifies the agent's performance and $\gamma$ is the discount factor that trades off current and future rewards. The goal is to implement a policy that maximises total discounted rewards, i.e., $\pi^\star(\bm a|\bm s) = \argmax_\pi V^\pi(\bm s)$,  where  $V^{\pi}(\bm{s}) = \mathbb{E}\Big[\sum_{t=0}^{\infty}\gamma^{t}\mathcal{R}(\bm{s}_{t},\bm{a}_{t}) \Big]$.

\subsection{Single Agent Soft Q-Learning}  
A way to constrain the behaviour of an agent is to modify the feasibility set of allowable policies. This can be achieved by introducing a constraint, such as a KL between two policy distributions, to the reinforcement learning objective. Such an approach has been already used within single agent reinforcement learning. For example, soft Q-learning has been used to reduce the overestimation problem of standard Q-learning~\cite{fox2016taming} and for building flexible energy-based policies in continuous domains~\cite{haarnoja2017reinforcement}. Most of these approaches modify the standard objective of reinforcement learning to
\begin{align}&\max_{\pi} \mathbb{E}\Big[\sum_{t=0}^{\infty} \gamma^{t}\mathcal{R}(\bm{s}_{t},\bm{a}_{t}) \Big] \nonumber \\
&\text{s.t.} \ \  \sum_{t=0}^{\infty} \mathbb{E}\left[\gamma^{t}\text{KL}\left(\pi(\bm{a}_{t}|\bm{s}_{t})||\rho(\bm{a}_{t}|\bm{s}_{t})\right)\right] \le C, \label{eq:softQlearning}
\end{align}
where $C$ is the amount of bits (or nats if using the natural logarithm) measured by the KL divergence that the policy $\pi$ is allowed to deviate from a reference policy $\rho$. The expectation operation is over state-action trajectories. 

To solve the above constrained problem, one typically introduces a Lagrange multiplier, $\beta$, and rewrites an equivalent unconstrained problem
\begin{align*}
\mathcal{V}^{\star}(\bm{s}) &= \max_{\pi} \mathbb{E}\bigg[\sum_{t=0}^{\infty} \gamma^t\bigg(\mathcal{R}(\bm{s}_{t},\bm{a}_{t})- \frac{1}{\beta}\log \frac{\pi(\bm{a}_{t}|\bm{s}_{t})}{\rho(\bm{a}_{t}|\bm{s}_{t})}\bigg)\bigg].
\end{align*}

To derive an algorithm for solving the above, one comes to recognise that $\mathcal{V}^{\star}(\bm{s})$ also satisfies a recursion similar to that introduced by the Bellman equations~\cite{puterman1994markov}. Additionally, the optimal policy can be written in closed form as
\begin{equation*}
	\pi^{\star}(\bm{a}|\bm{s}) = \frac{\rho(\bm{a}|\bm{s}) e^{\beta \mathcal{Q}^{\star}(\bm{s},\bm{a})}}{\sum_{\bm{a} \in \mathcal{A}}\rho(\bm{a}|\bm{s})e^{\beta \mathcal{Q}^{\star}(\bm{s},\bm{a})}},
\end{equation*}
where $\mathcal{Q}^{\star}(\bm{s},\bm{a}) := \mathcal{R}(\bm{s},\bm{a}) + \sum_{\bm{s}^{\prime} \in \mathcal{S}} \mathcal{T}(\bm{s}^{\prime}|\bm{s},\bm{a}) \mathcal{V}^{\star}(\bm{s}^{\prime})$, and $\bm{s}^{\prime} \in \mathcal{S}$. Notice that the above represents a generalisation of standard RL settings, where $\beta \rightarrow \infty $ corresponds to a perfectly rational valuation ($\mathcal V_{\beta \rightarrow \infty}^\star (\bm s)=\max_\pi V^\pi(\bm s)$), while for $\beta \rightarrow 0 $ we recover the valuation under $\rho$ ($\mathcal V_{\beta \rightarrow 0}^\star (\bm s)=V^\rho(\bm s)$). Clearly, we can generate a continuum of policies between the reference and the perfectly rational policy that maximises the expected reward by tuning the choice of $\beta$ as detailed in~\cite{leibfried2017information}.

\subsection{Two-Player Stochastic Games}
In two-player stochastic games~\cite{shapley1953stochastic,littman1994markov}, two agents, that we denote as the player and the opponent, are interacting in an environment. Each agent executes a policy that we write as $\pi_{\text{pl}}$ and 
$\pi_{\text{op}}$. At some time step $t$, the player chooses an action $\bm{a}^{\text{pl}}_{t} \sim \pi_{\text{pl}}\big(\bm{a}^{\text{pl}}_{t}|\bm{s}_{t}\big)$, while the opponent picks $\bm{a}^{\text{op}}_{t} \sim \pi_{\text{op}}\big(\bm{a}^{\text{op}}_{t}|\bm{s}_{t}\big)$. Accordingly, the environment transitions to a successor state $\bm{s}_{t+1} \sim \mathcal{T}_{\mathcal{G}}\big(\bm{s}_{t+1}|\bm{s}_{t},\bm{a}^{\text{pl}}_{t},\bm{a}^{\text{op}}_{t}\big)$, where $\mathcal{T}_{\mathcal{G}}$ denotes the joint transition model for the game. After transitioning to a new state, both agents receive a particular reward depending on the type of game  considered. In team games, both the player and the opponent maximise the same reward function $\mathcal{R}_{\mathcal{G}}\big(\bm{s}_{t},\bm{a}^{\text{pl}}_{t},\bm{a}^{\text{op}}_{t}\big)$. For zero-sum games, the player seeks to maximise $\mathcal{R}_{\mathcal{G}}$, whereas the opponent seeks to find a minimum. We write the policy dependent value as $V^{\pi_{\text{pl}}\pi_{\text{op}}}(\bm s) := \mathbb E \big[ \sum_{t=0}^\infty \gamma^t \mathcal{R}_{\mathcal{G}}\big(\bm{s}_{t},\bm{a}^{\text{pl}}_{t},\bm{a}^{\text{op}}_{t}\big) \big]$ where, in contrast to the one-player setting, the expectation is over state and \emph{joint-action} trajectories.

In stochastic games it is common to assume perfect rationality for both agents i.e., in the case of a zero-sum game the player computes the optimal value of state $\bm s$ as $V^{\star\text{pl}}_{\text{zs}}(\bm s) = \max_{\pi_{\text{pl}}} \min_{\pi_{\text{op}}} V^{\pi_{\text{pl}}\pi_{\text{op}}}(\bm s)$, while the opponent  as $V^{\star \text{op}}_{\text{zs}}(\bm s) =  \min_{\pi_{\text{op}}} \max_{\pi_{\text{pl}}} V^{\pi_{\text{pl}}\pi_{\text{op}}}(\bm s)$. Similarly, in team games the optimal value for the player is  $V^{\star \text{pl}}_{\text{tg}}(\bm s) = \max_{\pi_{\text{pl}}} \max_{\pi_{\text{op}}} V^{\pi_{\text{pl}}\pi_{\text{op}}}(\bm s)$ and for the opponent $V^{\star\text{op}}_{\text{tg}}(\bm s) =  \max_{\pi_{\text{opp}}} \max_{\pi_{\text{pl}}} V^{\pi_{\text{pl}}\pi_{\text{op}}}(\bm s)$. Although it is straightforward to show that for team games $V^{\star \text{op}}_{\text{tg}}(\bm s)  = V^{\star\text{pl}}_{\text{tg}}(\bm s) $, an important classic result in game theory  -- the minimax theorem \cite{osborne1994course} -- states that for zero-sum games $V^{\star \text{op}}_{\text{zs}}(\bm s)  = V^{\star \text{pl}}_{\text{zs}}(\bm s) $, i.e both team and zero sum games have a unique value.

Importantly, in complex games with large state-spaces the $\max$ and the $\min $ operations over all available policies are extremely difficult to compute. Humans and suboptimal agents seek to approximate these operations as best they can but never fully do so due to the lack of computational resources~\cite{ortega2016human}, approximations and introduced biases~\cite{lieder2012burn}.   
This limits the applicability of SG when interacting with suboptimal entities, e.g., in computer games when competing against human players. We next provide the first extension, to the best of our knowledge, of soft Q-learning to SGs and show how our framework can be used within the context of balancing the game's difficulty.


\section{Two-Player Soft Q-Learning}\label{sec:two-player_stochastic_games_with_KL}
To enable soft Q-learning in two-player games we introduce two KL constraints that allow us to separately control the performance of both agents. In particular, we incorporate a constraint similar to \eqref{eq:softQlearning} into the objective function for each agent and apply the method of Lagrange multipliers 
\begin{align}
\label{eq:policy_dependent_valueBR}
& \mathcal{V}^{\pi_{\text{pl}} \pi_{\text{op}}}(\bm s) = \mathbb{E}\bigg[\sum_{t=0}^{\infty} \gamma^{t} \bigg(\mathcal{R}_{\mathcal{G}}\Big(\bm{s}_{t},\bm{a}_{t}^{\text{pl}},\bm{a}_{t}^{\text{op}}\Big) \\ \nonumber
&-\frac{1}{\beta_{\text{pl}}} \log \frac{\pi_{\text{pl}}(\bm{a}_{t}^{\text{pl}}|\bm{s}_{t})}{\rho_{\text{pl}}(\bm{a}_{t}^{\text{pl}}|\bm{s}_{t})} - \frac{1}{\beta_{\text{op}}} \log \frac{\pi_{\text{op}}(\bm{a}_{t}^{\text{op}}|\bm{s}_{t})}{\rho_{\text{op}}(\bm{a}_{t}^{\text{op}}|\bm{s}_{t})}
\bigg)  
\bigg],
\end{align}
where the expectation is over joint-action trajectories, $\frac{1}{\beta_{\text{pl}}} \log \frac{\pi_{\text{pl}}(\bm{a}_{t}^{\text{pl}}|\bm{s}_{t})}{\rho_{\text{pl}}(\bm{a}_{t}^{\text{pl}}|\bm{s}_{t})}$ is the information cost for the player (that turns into a KL divergence with the expectation operator), and $\frac{1}{\beta_{\text{op}}} \log \frac{\pi_{\text{op}}(\bm{a}_{t}^{\text{op}}|\bm{s}_{t})}{\rho_{\text{op}}(\bm{a}_{t}^{\text{op}}|\bm{s}_{t})}$  is the information cost for the opponent.  The Lagrange multipliers $\beta_{\text{pl}}$ and $\beta_{\text{op}}$ are tuneable  parameters that we can vary at will. The distributions $\rho_{\text{pl}}(\bm{a}_{t}^{\text{pl}}|\bm{s}_{t})$ and $\rho_{\text{op}}(\bm{a}_{t}^{\text{op}}|\bm{s}_{t})$ are the arbitrary reference policies that we assume to be uniform\footnote{Please note considering other reference policies is left as an interesting direction for future work.}. Using the above, the player and the opponent compute optimal soft-value of a state $\bm s$ using 
\begin{align}
\label{eq:optimization_freeenergy_player}
&\mathcal{V}^{\star}_{\text{pl}}(\bm{s}) = \max_{\pi_{\text{pl}}} \ext_{\pi_{\text{op}}} \mathcal{V}^{\pi_{\text{pl}} \pi_{\text{op}}}(\bm s), \ \mathcal{V}^{\star}_{\text{op}}(\bm{s}) =  \ext_{\pi_{\text{op}}} \max_{\pi_{\text{pl}}} \mathcal{V}^{\pi_{\text{pl}} \pi_{\text{op}}}(\bm s).
\end{align}
We define the extremum operator $\ext$ to correspond to a $\max$ in the case of positive $\beta_{\text{op}}$ and to a $\min$ in the case of negative $\beta_{\text{op}}$.  

It is clear that this novel formulation of the optimisation problems in Equations~\eqref{eq:optimization_freeenergy_player} generalise to cover both zero-sum and team games depending on the choice of $\beta_{\text{op}}$. By fixing $\beta_{\text{pl}} \rightarrow \infty$ and  setting $\beta_{\text{op}} \rightarrow -\infty$ or  $\beta_{\text{op}} \rightarrow \infty$ we recover, respectively, a zero-sum or a team game with perfectly rational agents. For $\beta_{\text{op}} \rightarrow 0$ we derive a game by which the opponent simply employs policy $\rho_{\text{op}}$. For finite values of $\beta_{\text{op}}$, we obtain a continuum of opponents with bounded performance ranging from fully adversarial to fully collaborative including a random policy. It is important to note, as we will show later, that the analytical form of the optimal policies that solve \eqref{eq:optimization_freeenergy_player} are independent of the extremum operator and only depend on the parameters $\beta_{\text{pl}}$ and $\beta_{\text{op}}$.

\subsection{Unique Value for Two-Player Soft Q-Learning}
In this section we show that the equations in~\eqref{eq:optimization_freeenergy_player} are equivalent, $\mathcal V_{\text{pl}}^\star (\bm s) =  \mathcal V_{\text{op}}^\star (\bm s)$, for any $\beta_\text{pl}$ and $\beta_\text{op}$, that is, our two player soft Q-learning exhibit a unique value. 
 
We start by defining the free energy operator as 
\begin{align}\label{eq:f}
f(\pi_{\text{pl}}&,\pi_{\text{op}},\bm{s},  \mathcal{V})   := \mathbb{E}_{\pi_{\text{pl}},\pi_{\text{op}}}\bigg[ \mathcal{R}_{\mathcal{G}}(\bm{s},\bm{a}^{\text{pl}},\bm{a}^{\text{op}})  + \gamma \mathbb{E}_{\mathcal{T}_{\mathcal{G}}}\left[\mathcal{V}(\bm s') \right] \nonumber \\
& -\frac{1}{\beta_{\text{pl}}} \log \frac{\pi_{\text{pl}}(\bm{a}^{\text{pl}}|\bm{s})}{\rho_{\text{pl}}(\bm{a}^{\text{pl}}|\bm{s})} -\frac{1}{\beta_{\text{opp}}} \log \frac{\pi_{\text{op}}(\bm{a}^{\text{op}}|\bm{s})}{\rho_{\text{op}}(\bm{a}^{\text{op}}|\bm{s})} \bigg]
\end{align}
for an arbitrary free energy vector $\mathcal V$. 
Then the Bellman-like operators for both the player and the opponent can be expressed as:
\begin{align}
\label{Eq:Bla}
\mathcal B_{\text{pl}} \mathcal{V}(\bm s) &= \max_{\pi_{\text{pl}}} \ext_{\pi_{\text{op}}} f(\pi_{\text{pl}},\pi_{\text{op}},\bm{s},\mathcal{V}) \\\nonumber
\mathcal B_{\text{op}} \mathcal{V}  (\bm s) &=  \ext_{\pi_{\text{op}}} \max_{\pi_{\text{pl}}} f(\pi_{\text{pl}},\pi_{\text{op}},\bm{s},\mathcal{V}).
\end{align}
\textbf{Proof Sketch:} For proving our main results, summarised in Theorem~\ref{col:optimal_freeenergies_are_equal}, we commence by showing that the equations in~\eqref{eq:optimization_freeenergy_player} are equivalent. This is achieved by showing that the two operators in Equation~\eqref{Eq:Bla} are in fact equivalent, see Lemma~\ref{lem:equality_operators}. Proving these operators to be contractions converging to a unique fixed point (see Theorem~\ref{theo:contraction}), we conclude that $\mathcal V_{\text{pl}}^\star (\bm s) =  \mathcal V_{\text{op}}^\star (\bm s)$  (see Appendix for proof details).

\begin{lemma}\label{lem:equality_operators}
For any $\beta_{\text{pl}} \in \mathbb R$ and  $\beta_{\text{op}} \in \mathbb R$, and arbitrary free energy vector $\mathcal V$, then
$\mathcal B_{\text{pl}} \mathcal{V} (\bm s) =  \mathcal B_{\text{op}} \mathcal{V}(\bm s)$.
\end{lemma}

Due to Lemma~\ref{lem:equality_operators}, we can define the generic operator $\mathcal B \mathcal{V}(\bm s):= \mathcal B_{\text{pl}} \mathcal{V} (\bm s) =  \mathcal B_{\text{op}} \mathcal{V} (\bm s)$. Then, for this generic operator, we can prove the following. 
\begin{theorem}[Contraction]\label{theo:contraction}
For $\beta_{\text{op}} \in \mathbb R$ and $\beta_{\text{pl}} \in \mathbb R$, the operator $\mathcal{B}$ is an $L_\infty$-norm contraction map $\| \mathcal{B} \mathcal{V} - \mathcal{B} \mathcal{\bar{V}} \|_\infty  \le \gamma \| \mathcal{V} - \mathcal{\bar{V}} \|_\infty$, where $\mathcal{V}$ and $\mathcal{\bar{V}}$ are two arbitrary free energy vectors and $\gamma$ is the discount factor.
\end{theorem}
Note that our reward is policy dependent (in the information cost) and, therefore, Theorem~\ref{theo:contraction} is not a direct consequence of known results \cite{littman1996generalized}, which assume that these rewards are policy independent. Using the above and the Banach's fixed point theorem~\cite{puterman1994markov}, we obtain the following corollary.
\begin{corollary}[Unique fixed point]\label{col:unique_fixed_point}
 The contraction mapping $\mathcal{B}$ exhibits a unique fixed-point $\mathcal{V}^{\star}$ such that $\mathcal{B}\mathcal{V}^{\star} = \mathcal{V}^{\star}$.
\end{corollary}
Due to Lemma~\ref{lem:equality_operators}, Theorem~\ref{theo:contraction} and Corollary~\ref{col:unique_fixed_point}, we arrive at the following.  
\begin{corollary}\label{col:optimal_freeenergies_are_equal}
	Two-player stochastic games with soft Q-learning have a unique value, i.e. $\mathcal V_{\text{pl}}^\star (s) =  \mathcal V_{\text{op}}^\star (s)$.
\end{corollary}

\subsection{Bounded-Optimal Policies}  
Corollary~\ref{col:optimal_freeenergies_are_equal} allows us to exploit the fact that there exists one unique value to generate the policies for both agents. With this in mind, we next design an algorithm (similar in spirit to standard Q-Learning) that acquires tuneable policies. We start by defining a state-action value function, in resemblance to the Q-function, as
\begin{align*}
\mathcal{Q}^{\star}(\bm{s},\bm{a}^{\text{pl}},\bm{a}^{\text{op}})  & :=  \mathcal{R}_{\mathcal{G}}(\bm{s}, \bm{a}^{\text{pl}},\bm{a}^{\text{op}})   + \gamma \mathbb E_{\mathcal{T}_{\mathcal{G}}} \left[ \mathcal{V}^{\star}(\bm{s}^{\prime})\right].
\end{align*}
For action selection, neither the player nor the opponent can directly use $\mathcal Q^\star$ as it depends on the action of the other agent, which is unknown a priori. Instead, it can be shown that agents must first compute the certainty equivalent by marginalising  $\mathcal Q^\star$ as 
\begin{align*} 	
	&\mathcal{Q}^{\star}_{\text{pl}}  (\bm{s},  \bm{a}^{\text{pl}}) := \frac{1}{\beta_{\text{op}}} \log \sum_{\bm{a}^{\text{op}}} \rho_{\text{op}}(\bm{a}^{\text{op}}|\bm{s}) \exp \left(\beta_{\text{op}} \mathcal{Q}^{\star}(\bm{s},\bm{a}^{\text{pl}},\bm{a}^{\text{op}}) \right)\\ \nonumber 
	&\mathcal{Q}^{\star}_{\text{op}}  (\bm{s}, \bm{a}^{\text{op}}) := \frac{1}{\beta_{\text{pl}}} \log \sum_{\bm{a}^{\text{pl}}} \rho_{\text{pl}}(\bm{a}^{\text{pl}}|\bm{s}) \exp \left(\beta_{\text{pl}} \mathcal{Q}^{\star}(\bm{s},\bm{a}^{\text{pl}},\bm{a}^{\text{op}}) \right). \nonumber 
\end{align*}
With these definitions and using standard variational calculus, we obtain optimal policies for both the player and the opponent as\footnote{Note that, if we assume that the action space has low cardinality, $\mathcal{Q}^{\star}_{\text{pl}} (\bm{s},\bm{a}^{\text{pl}})$ and  $\mathcal{Q}^{\star}_{\text{op}}  (\bm{s},\bm{a}^{\text{op}})$ can be computed exactly.
}  
\begin{align}\label{eq:player_policy}
\pi^{\star}_{\text{pl}}(\bm{a}^{\text{pl}} |\bm{s}) &= \argmax_{\pi_{\text{pl}}} \ext_{\pi_{\text{op}}} f(\pi_{\text{pl}},\pi_{\text{op}},\bm{s},\mathcal{V}^\star)\nonumber \\
& = \frac{1}{Z_{\text{pl}}(\bm s)} \rho_{\text{pl}}(\bm{a}^{\text{pl}}|\bm{s})\exp\big(\beta_{\text{pl}} \mathcal{Q}^{\star}_{\text{pl}}(\bm{s},\bm{a}^{\text{pl}}) \big)  \\ \nonumber
 \pi^{\star}_{\text{op}}(\bm{a}^{\text{op}}|\bm{s}) &=  \argext_{\pi_{\text{op}}} \max_{\pi_{\text{pl}}} f(\pi_{\text{pl}},\pi_{\text{op}},\bm{s},\mathcal{V}^\star) \nonumber \\
& = \frac{1}{Z_{\text{op}}(\bm s)} \rho_{\text{op}}(\bm{a}^{\text{op}}|\bm{s})\exp\Big( \beta_{\text{op}} \mathcal{Q}^{\star}_{\text{op}}(\bm{s},\bm{a}^{\text{op}})\Big),  \nonumber
\end{align}
where $Z_{\text{pl}}(\bm s)$ and $Z_{\text{op}}(\bm s)$ are normalising functions which can be exactly computed when assuming small discrete action spaces. 

Hence, $\mathcal V^\star (\bm s)$ can be expressed in closed form by incorporating the optimal policies in Equation~\eqref{eq:optimization_freeenergy_player} giving
\begin{equation}\label{eq:bounded_rational_value}
	\mathcal V^\star(\bm s) =  \frac{1}{\beta_{\text{pl}}} \log \sum_{\bm{a}^{\text{pl}}} \rho_{\text{pl}}(\bm{a}^{\text{pl}}|\bm{s}) \exp \left(\beta_{\text{pl}} \mathcal{Q}^{\star}_{\text{pl}} (\bm{s},\bm{a}^{\text{pl}}) \right).
\end{equation}

As summarised in Algorithm~\ref{Alg:tabular}, we learn $\mathcal{Q}^{\star}(\bm{s},\bm{a}^{\text{pl}},\bm{a}^{\text{op}})$ by applying the following recursion rule:
\begin{align}
&\mathcal{Q}_{k+1}(\bm{s},\bm{a}^{\text{pl}},\bm{a}^{\text{op}}) = \mathcal{Q}_{k}(\bm{s},\bm{a}^{\text{pl}},\bm{a}^{\text{op}}) \label{eq:q_update} \\
&+  \alpha \big(\mathcal{R}_{\mathcal{G}} (\bm{s},\bm{a}^{\text{pl}},\bm{a}^{\text{op}}) + \gamma \mathcal{V}_{k}(\bm{s}^{\prime}) - \mathcal{Q}_{k} (\bm{s},\bm{a}^{\text{pl}},\bm{a}^{\text{op}}) \big). \nonumber
\end{align}
Here, $\alpha$ is the learning rate, $k$ the learning step, and $\mathcal{V}_{k}(\bm{s}^{\prime})$ is computed as in Equation~\eqref{eq:bounded_rational_value} using the current $\mathcal Q_k$ estimate.

\begin{algorithm}\caption{Two-Player Soft Q-Learning}
\begin{algorithmic}[1]
\STATE Given $\rho_{\text{pl}}$, $\rho_{\text{op}}$, $\beta_{\text{pl}}$, $\beta_{\text{op}}$, $\mathcal A$, $\mathcal S$ and learning rate $\alpha$
\STATE $\mathcal{Q} (\bm s, \bm a^{\text{pl}}, \bm a^{\text{op}}) \gets 0$ 
\WHILE {not converged}
\STATE Collect transition $\big(\bm s_t, \bm a^{\text{pl}}_t, \bm a^{\text{op}}_t, \mathcal{R}_{\text{t}}, \bm{s'}_t \big) $, where $\bm{a}^{\text{pl}} \sim \pi_{\text{pl}}$, $\bm{a}^{\text{op}} \sim \pi_{\text{op}}$ and $\mathcal R_t$ is the reward at time $t$. 
\STATE Update $\mathcal Q$ according to Equation~\eqref{eq:q_update}
\ENDWHILE

\STATE {\bf return} $\mathcal{Q} (\bm s, \bm a^{\text{pl}}, \bm a^{\text{op}})$

\end{algorithmic}
\label{Alg:tabular}
\end{algorithm}

\begin{figure*}[t]
\centering
\subfigure[High Rationality]{
	\label{fig:HighRation}
\includegraphics[height=0.18\textwidth,width=0.35\textwidth, trim={0 0 0 0}]{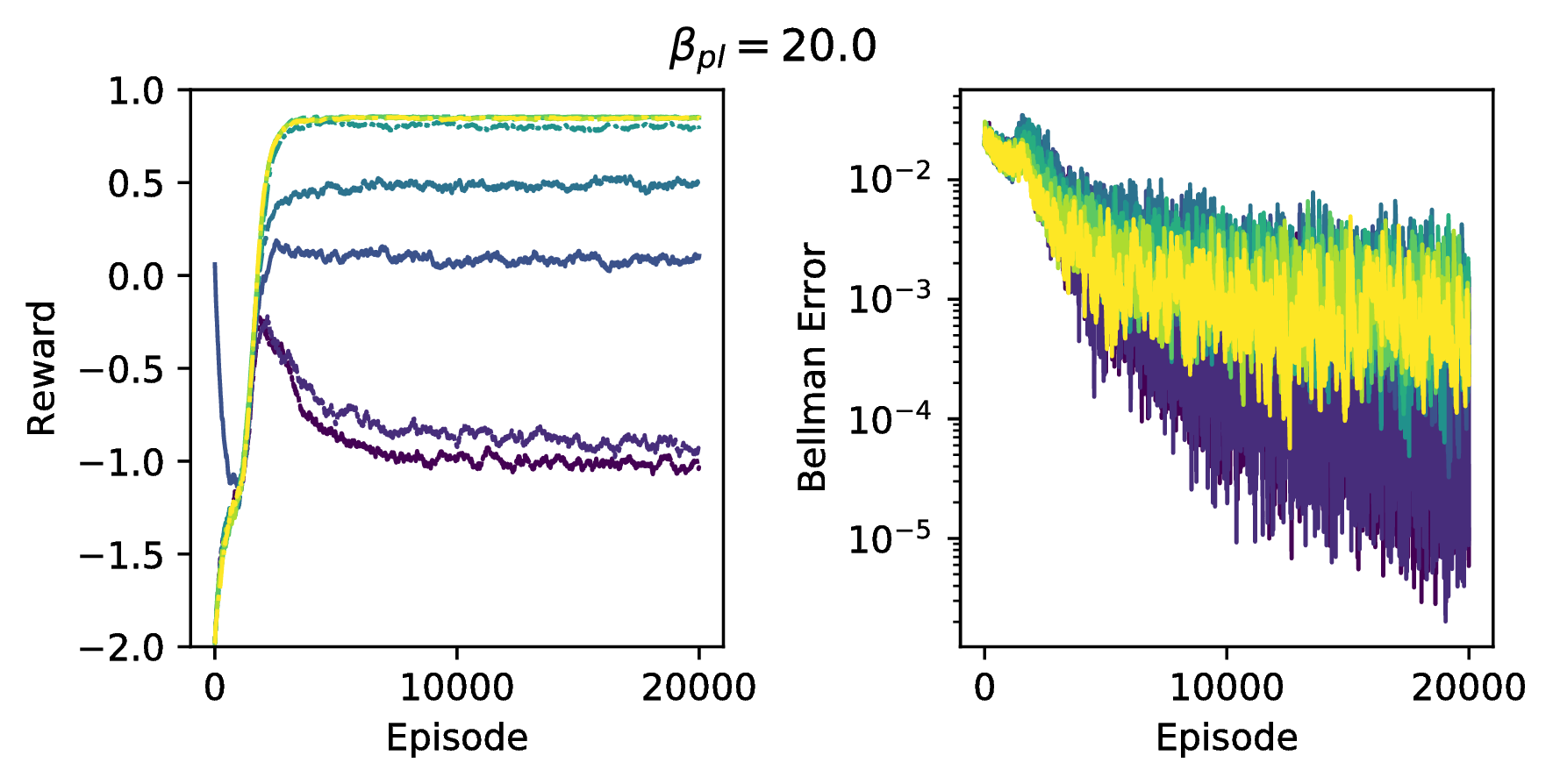}
}
\hfill\hspace{-5.0em}\hfill
\subfigure[Low Rationality]{
	\label{fig:LowRation}
\includegraphics[height=0.18\textwidth,width=0.41\textwidth]{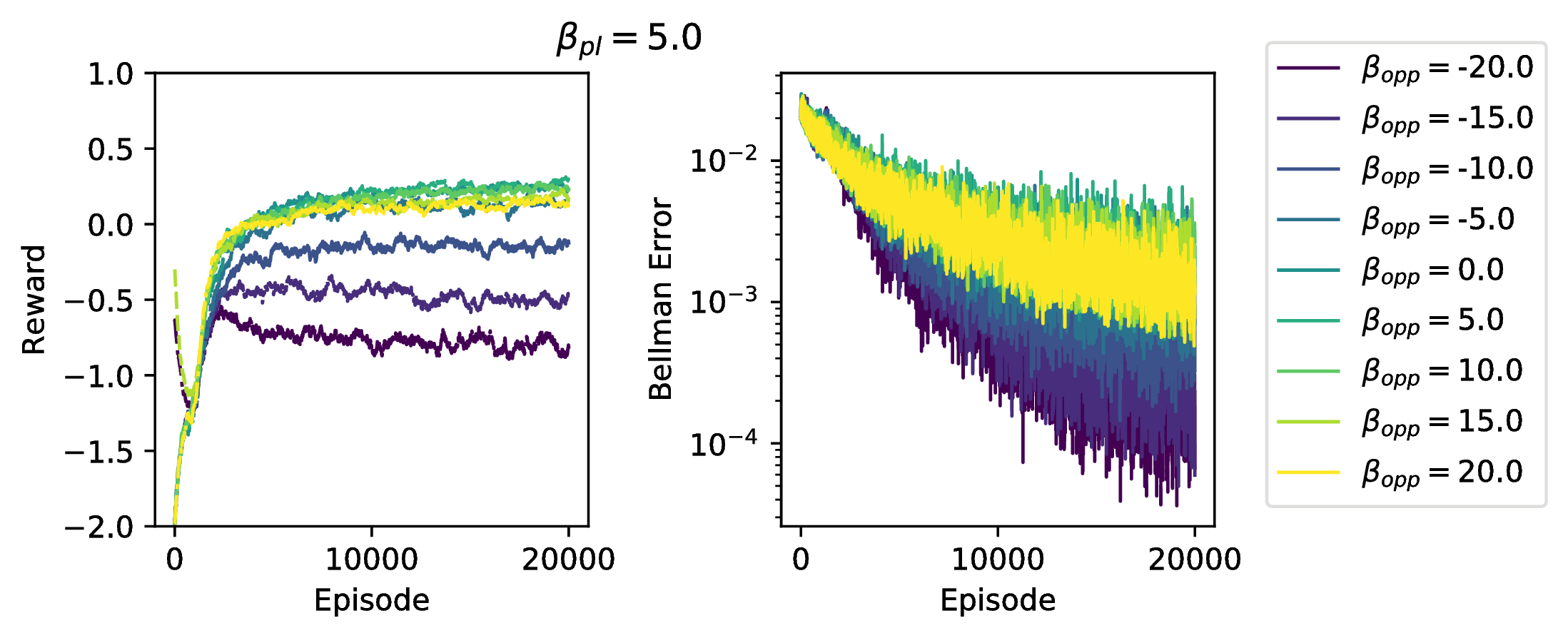}
}
\hfill\hspace{-3.1em}
\subfigure[Broad Range]{
	\label{fig:HeatMap}
\includegraphics[height=0.18\textwidth,width=0.20\textwidth]{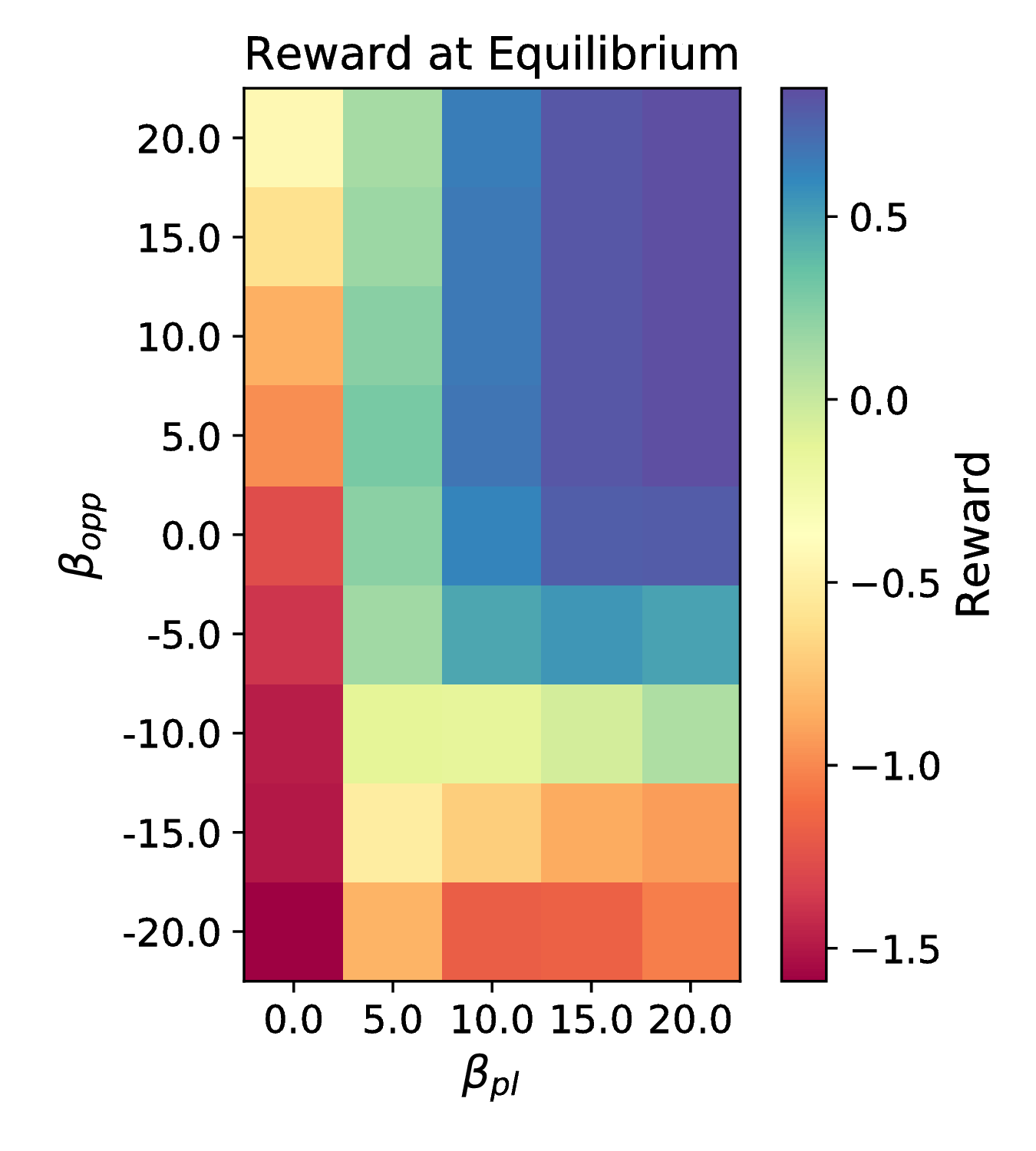}
}
\hfill
\vspace{-1.2em}
\caption{Evolution of reward and Bellman error during training for different $\beta_\text{pl}$ and $ \beta_\text{op}$. We vary $\beta_\text{op}$ while fixing $\beta_\text{pl}=20$ in panels (a), and $\beta_\text{pl} = 5.0 $ in panels (b). The heat map in (c) visualises these rewards for a broader range of parameters. These results confirm that our approach can modulate performance.}
\vspace{-.6em}
\end{figure*}

\section{Real-World Considerations} \label{sec:real_world}
Two restrictions limit the applicability of our algorithm to real-world scenarios. First, Algorithm~\ref{Alg:tabular} implicitly assumes the knowledge of the opponent's  parameter $\beta_{\text{op}}$. Obtaining $\beta_{\text{op}}$ in real-world settings can prove difficult. Second, our algorithm has been developed for low-dimensional state representations. Clearly, this restricts its applicability to high-dimensional states that are typical to computer games. 

To overcome these issues, we next develop an online maximum likelihood procedure to infer $\beta_\text{op}$ from data gathered through the interaction with the opponent, and then generalise Algorithm~\ref{Alg:tabular} to high-dimensional representations by proposing a deep learning architecture.

\subsection{Estimating $\beta_\text{op}$ \& Game-Balancing}\label{Sec:ML}
Rather than assuming access to the opponents rationality parameter, we next devise a maximum likelihood estimate that allows the agent to infer (in an online fashion) about $\beta_{\text{op}}$ and consequently, about the real policy of the opponent (through Equation~\eqref{eq:player_policy}).

Contrary to current SG techniques that attempt to approximate the opponent's policy directly, our method allows to reason about the opponent by only approximating a one dimensional parameter, i.e., $\beta_{\text{op}}$ in Equation~\eqref{eq:player_policy}\footnote{Please note that similar to the previous section we assume the opponent's reference policy $\rho_\text{op}$ to be uniform. This, however, does not impose a strong restriction since having a uniform reference policy enables enough flexibility to model various degrees of the opponent's performances (see Section~\ref{Sec:Experiments}). 
}. \\
\textbf{Estimating $\beta_{\text{op}}$:} We frame the problem of estimating $\beta_{\text{op}}$ as a one of online maximum likelihood estimation. Namely, we assume that the player interacts in $R$ rounds with the opponent. At each round, $j$, the player gathers a dataset of the form $\mathcal{D}^{(j)} = \left\{\bm{s}^{(j)}_{i},\bm{a}_{i}^{(j),\text{pl}},\bm{a}_{i}^{(j),\text{op}}\right\}_{i=1}^{m^{(j)}}$ with $m^{(j)}$ denoting the total number of sampled transitions during round $j$. Given $\mathcal{D}^{(j)}$, the agent estimates its knowledge of the opponent's model i.e., $\beta_{\text{op}}$ by solving the following problem\footnote{Please note that this problem can be easily solved using stochastic gradient descent.}
\begin{equation*}
\max_{\beta_{\text{op}}} \log \textrm{Pr}\left(\mathcal{D}^{(j)}\right) := \max_{\beta_{\text{op}}} \sum_{i=1}^{m^{(j)}} \log \pi_{\text{op}}^{\star}\left(\bm{a}_{i}^{(j),\text{op}}|\bm{s}_{i}^{(j)}\right), 
\end{equation*}
where $\pi_{\text{op}}^{\star}\left(\bm{a}_{i}^{(j),\text{op}}|\bm{s}_{i}^{(j)}\right)$ is defined in Equation~\eqref{eq:player_policy}. As rounds progress, the agent should learn to improve its estimate of $\beta_{\text{op}}$. Such an improvement is quantified, in terms of regret\footnote{Regret is a standard notion to quantify the performance of an online learning algorithm. Regret measures the performance of the agent with respect to an adversary that has access to all information upfront. }, in the following theorem for both a fixed and a time-varying opponent.
\begin{theorem}
After $R$ rounds, the average static-regret for estimating $\beta_{\text{op}}$ vanishes as: 
\begin{equation}
\label{EQ:Static}
\sum_{j=1}^{R} \mathcal{L}_{j}(\beta^{(j)}_{\text{op}}) - \min_{u} \Big[\sum_{j=1}^{R}\mathcal{L}_{j}(u) \Big] \approx \mathcal{O}(\sqrt{R})
\end{equation}
For a time-varying opponent, the dynamic regret bound dictates: 
\begin{align*}
\sum_{j=1}^{R} \mathcal{L}_{j}(\beta^{(j)}_{\text{op}}) -  \sum_{j=1}^{R}\min_{u_{j}}\mathcal{L}_{j}(u_{j})  &\approx  \mathcal{O}\Bigg(\sqrt{R} \\
&\left(1 + \sum_{j=1}^{R-1}||u_{j+1}^{\star} - u_{j}^{\star}||_{2}^{2}\right)\Bigg), 
\end{align*}
with $\mathcal{L}_{j}(\cdot)$ denoting the negative of the log-likelihood and $u_{j}^{\star} = \arg\min_{u} \mathcal{L}_{j}(u)$. 
\end{theorem}
From the above theorem we conclude that against a fixed-opponent our method guarantees correct approximation of $\beta_{\text{op}}$. This is true since the average regret, $\mathcal{O}(\sqrt{R})/R$, vanishes as $R\rightarrow \infty$. When it comes to a dynamic opponent, however, it is clear that our bound depends on how the value of the opponents multiplier parameter (in other words its policy) vary with in terms of rounds. In case these variations are bounded in number, we can still guarantee vanishing regrets. If not, the regret bound can grow arbitrarily large since $\sum_{j=1}^{R-1}||u_{j+1}^{\star} - u_{j}^{\star}||_{2}^{2}$ can introduce a factor $R$.

\textbf{Game Balancing:} Now that we have a way to learn $\mathcal Q^\star$ and estimate $\beta_\text{op}$ simultaneously, we could balance the game using  the estimate of $\beta_\text{op}$ to adjust the player's parameter $\beta_\text{pl}$. A simple heuristic that proved successful in our experiments was to simply set $\beta_{\text{pl}} = |\beta_\text{op} | + \Delta$, where $\Delta$ denotes an additional performance-level the player can achieve. Setting $\Delta = 0$ would correspond to agents with the same KL constraints, whereas setting $\Delta > 0$ would imply a stronger player with a softer KL constraint (see Section~\ref{Sec:ExpTwo}). 

\subsection{Deep Two-Player Soft Q-Learning}\label{Sec:Deep}
When tackling higher dimensional problems, one has to rely on function approximators to estimate the Q-function, or in our case, the function $\mathcal{Q}(\bm{s},\bm{a}^{\text{pl}},\bm{a}^{\text{op}})$. We borrow two ideas from deep Q-networks~\cite{mnih2015human} that allow us to stabilise learning with high-dimensional representations for our SG setting. First, we use the notion of a replay memory to store the following transitions $(\bm{s},\bm{a}^{\text{pl}},\bm{a}^{\text{op}},\bm{s}^{\prime},\mathcal{R}_{\mathcal{G}}(\cdot))$ and, second, we use a target network denoted by $\mathcal{Q}(\bm{s},\bm{a}^{\text{pl}},\bm{a}^{\text{op}};\bm{\theta}_{i}^{-})$ to handle non-stationarity of the objective. We learn $\mathcal{Q}(\bm{s},\bm{a}^{\text{pl}},\bm{a}^{\text{op}};\bm{\theta}_{i})$, by using a neural network that receives $\bm s$ as input and outputs a matrix of $\mathcal Q$-values for each combination of the agents' actions. The loss function that we seek to minimise is $\mathcal{L}(\bm{\theta}_{i}) = \mathbb{E} \big[\left(\mathcal{R}_{\mathcal{G}}(\bm{s},\bm{a}^{\text{pl}},\bm{a}^{\text{op}})+ \gamma \mathcal{V}(\bm{s};\bm{\theta}_{i}^{-}) - \mathcal{Q}(\bm{s},\bm{a}^{\text{pl}},\bm{a}^{\text{op}};\bm{\theta}_{i})\right)^{2}
\big]$

with the expectation taken over the distribution of transitions sampled from the replay memory, and $\mathcal{V}(\bm{s};\bm{\theta}_{i}^{-})$ computed as in Equation~\eqref{eq:bounded_rational_value}. Clearly, the above optimisation problem is similar to standard DQNs with the difference that error is measured between \emph{soft Q-values}. 

\section{Experiments}\label{Sec:Experiments}
We consider two cases in our experiments. The first assumes a low-dimensional setting, while the second targets the high-dimensional game of Pong. In both cases we consider full and no control of the opponent. Full control will allow us to validate our intuitions of tuneable behaviour, while the second sheds-the-light on the game balancing capabilities of Section~\ref{sec:real_world}. 

\subsection{Low-dimensional Experiments}

\textbf{The Setup:} We validate Algorithm~\ref{Alg:tabular} on a 5 $\times$ 6 grid-world, where we consider two agents interacting. Each can choose an action from $\mathcal{A} = \{\text{left},\text{right},\text{up}, \text{down}, \text{pick-up}\}$. The first four actions are primitive movements, while the last corresponds to picking-up an object when possible. The reward of the first player is set to $-0.02$ for any movement and to $+1$ for picking up the  object located in cell (2,6). 

The setting described in this paper allows for a range of games that can be continuously varied between cooperative and defective games depending on the choice of $\beta_{\text{op}}$ -- a setting not allowed by any of the current techniques to stochastic games. In other words, the goal of the opponent, now, depends on the choice of $\beta_{\text{op}}$. Namely, for positive values of $\beta_{\text{op}}$, the opponent is collaborative, whereas for negative $\beta_{\text{op}}$ it is adversarial. $\beta_{\text{op}}$ values in between correspond to tuneable performance varying between the above two extremes.  

We demonstrate adversarial behaviour by allowing agents to block each other either when trying to reach the same cell, or when attempting to transition to a cell previously occupied by the other agent. In such cases the respective agent remains in its current position. Given the determinism of the environment, a perfectly rational adversarial opponent can always impede the player to reach the goal. However, due to the KL constraints the opponent's policy becomes ``less'' aggressive, allowing the player to exploit the opponent's mistakes and arrive to the goal. For all experiments we used  a high learning rate of $\alpha = 0.5$\footnote{A deterministic environment transitions allows for a large learning rate.}.

\textbf{Tuning the Player's Performance:} To validate tuneablity, we assess the performance of the player when reaching convergence while varying $\beta_{\text{pl}}$ and $\beta_{\text{op}}$. In the first set of experiments, we fixed $\beta_{\text{pl}} =20$ and varied $\beta_{\text{op}} = \{-20, -15, -10, -5, 0, 5, 10, 15, 20\}$. We expect that the player obtains high reward for collaborative opponents ($\beta_{\text{op}}>0$) or highly sub-optimal adversarial opponents ($\beta_{\text{op}} \approx 0$), and low rewards for strong adversarial opponents ($\beta_{\text{op}} \ll 0$). Indeed, the results shown in Figure~\ref{fig:HighRation} confirm these intuitions.  

For a broader spectrum of analysis, we lower $\beta_{\text{pl}}$ from $20$ to $5$ and re-run the same experiments. Results in Figure~{\ref{fig:LowRation}} reaffirm the previous conclusions. Here, however, the player attains slightly lower rewards as $\beta_\text{pl}$ is decremented. Finally, in Figure~\ref{fig:HeatMap} we plot the reward attained after convergence for a broad range of parameter values. We clearly see the effect of the modulation in both parameters on the resultant reward. The best reward is achieved when both parameters have positive high values, and the least reward for the lowest values. \\
\textbf{Estimating $\beta_\text{op}$:} 
The goal of these experiments is to evaluate the correctness of our maximum likelihood estimate (Section~\ref{Sec:ML}) of $\beta_{\text{op}}$. To conduct these experiments, we fixed $\beta_{\text{pl}} = 10$ and generated data with  $\beta^{\star}_{\text{op}} = 5$ and $\beta^{\star}_{\text{op}} = -10$ that are unknown to the player. At each interaction with the environment, we updated $\mathcal{V}(\cdot)$ according to Algorithm~\ref{Alg:tabular} and $\beta_{\text{op}}$ using a gradient step in the maximum likelihood objective. Results reported in Figure~\ref{fig:exp4_learning_beta}, clearly demonstrate that our extension to estimate $\beta_{\text{op}}$ is successful\footnote{In the case where the opponent would have an arbitrary policy $\pi_a$ then $\beta_{\text{op}}$would converge to a value that attempts to make $\pi_{\text{op},\beta_{\text{op}}}$ as close as possible to $\pi_a$.}.

\subsection{High-dimensional Experiments}\label{Sec:ExpTwo}
We repeat the experiments above but now considering our deep learning architecture of Section~\ref{Sec:Deep} on the game of Pong. 

\textbf{The Setup:}
We use the game Pong  from the Roboschool package\footnote{https://github.com/openai/roboschool}. The state space is 13-dimensional i.e., x-y positions and x-y- velocities for both agents and the ball, and an additional dimension for time. We modified the action space to consist of $ |\mathcal A | = 9$ actions where the set corresponds to  $\mathcal A = \left\lbrace \text{left}, \text{stay}, \text{right} \right\rbrace \times \left\lbrace \text{up}, \text{stay}, \text{down} \right\rbrace$. We also modified the reward function to make it compatible with zero-sum games in such a way that if the player scores, the reward $\mathcal R_{\mathcal G}$ is set to $+1$, whereas if the opponent scores, to $-1$. The networks that represent soft Q-values, $\mathcal{Q}(\bm{s},\bm{a}^{\text{pl}}, \bm{a}^{\text{op}})$, are multilayer perceptrons composed of two hidden layers, each with $100$ units, and an a matrix output layer composed of $| \mathcal A \times \mathcal A | = 81$ units ($9\times 9$ actions). Here, each unit denotes a particular combination of $\bm{a}^{\text{pl}} \in \mathcal A$ and $\bm{a}^{\text{op}} \in \mathcal A$. After each hidden layer, we introduce a ReLU non-linearity. We used a learning rate of $10^{-4}$, the ADAM optimizer, a batch size of $32$, and updated the target every $30000$ training steps. \\
\textbf{Tuning the Player's Performance:} In this experiment, we demonstrate successful tuneable performance. Figure~\ref{fig:PongOne} shows that for a highly adversarial opponent (i.e., $\beta_{\text{op}} = -50$) the player ($\beta_\text{pl} = 20$) acquired negative rewards, whereas for a weak opponent or even collaborative, the player obtained high reward. Game-play videos can be found at https://sites.google.com/site/submission3591/. 

\begin{figure}[t]
    \centering
    \includegraphics[width = 0.95\columnwidth, height=3cm]{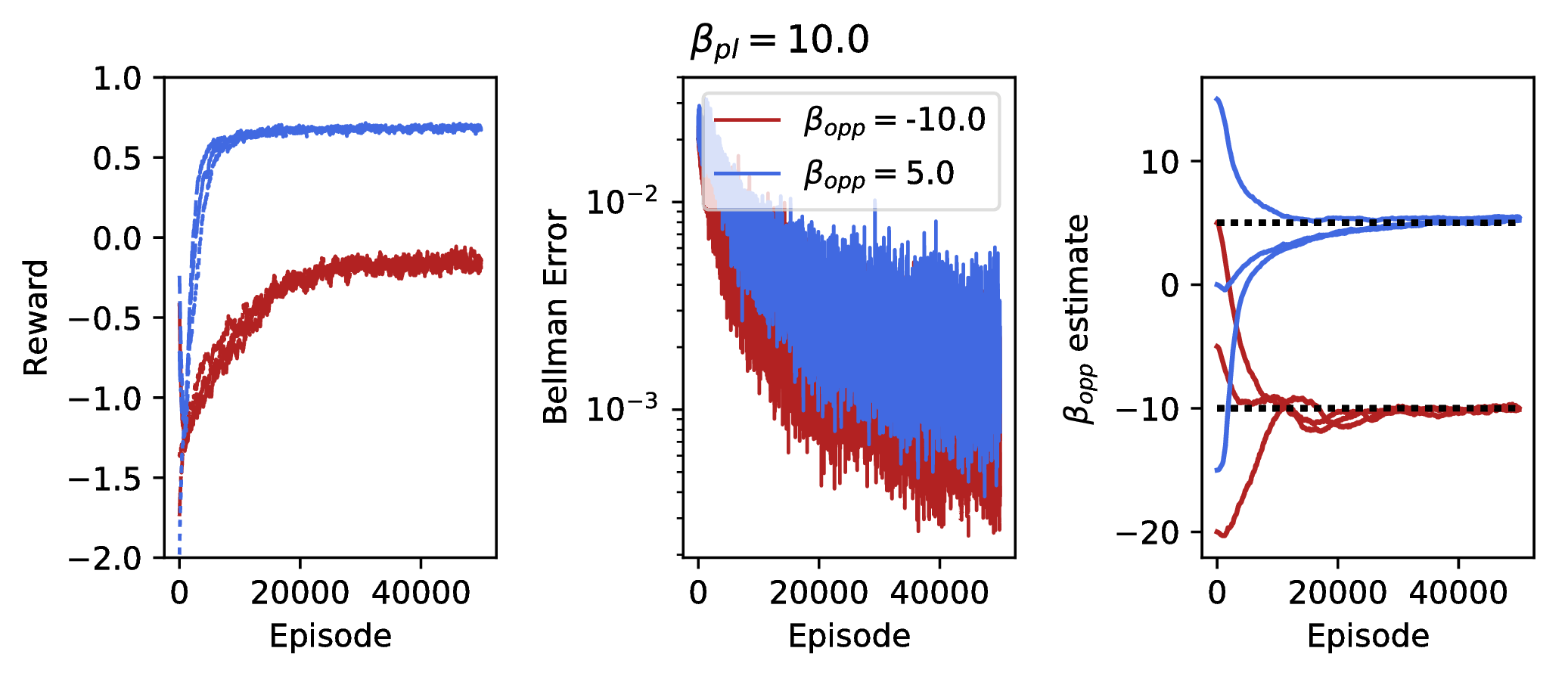}
    \caption{Reward, Bellman error and $\beta_{\text{op}}$ estimate over episodes. We see that the maximum likelihood estimator is capable of discovering the correct value for the red ($\beta_\text{op}^\star = -10$) and blue ($\beta_\text{op}^\star = 5$) opponents, from three different initial estimates of $\beta_{\text{op}}$.}
    \label{fig:exp4_learning_beta}
\end{figure}

\begin{figure}[t]
	\centering
	\includegraphics[scale=0.105]{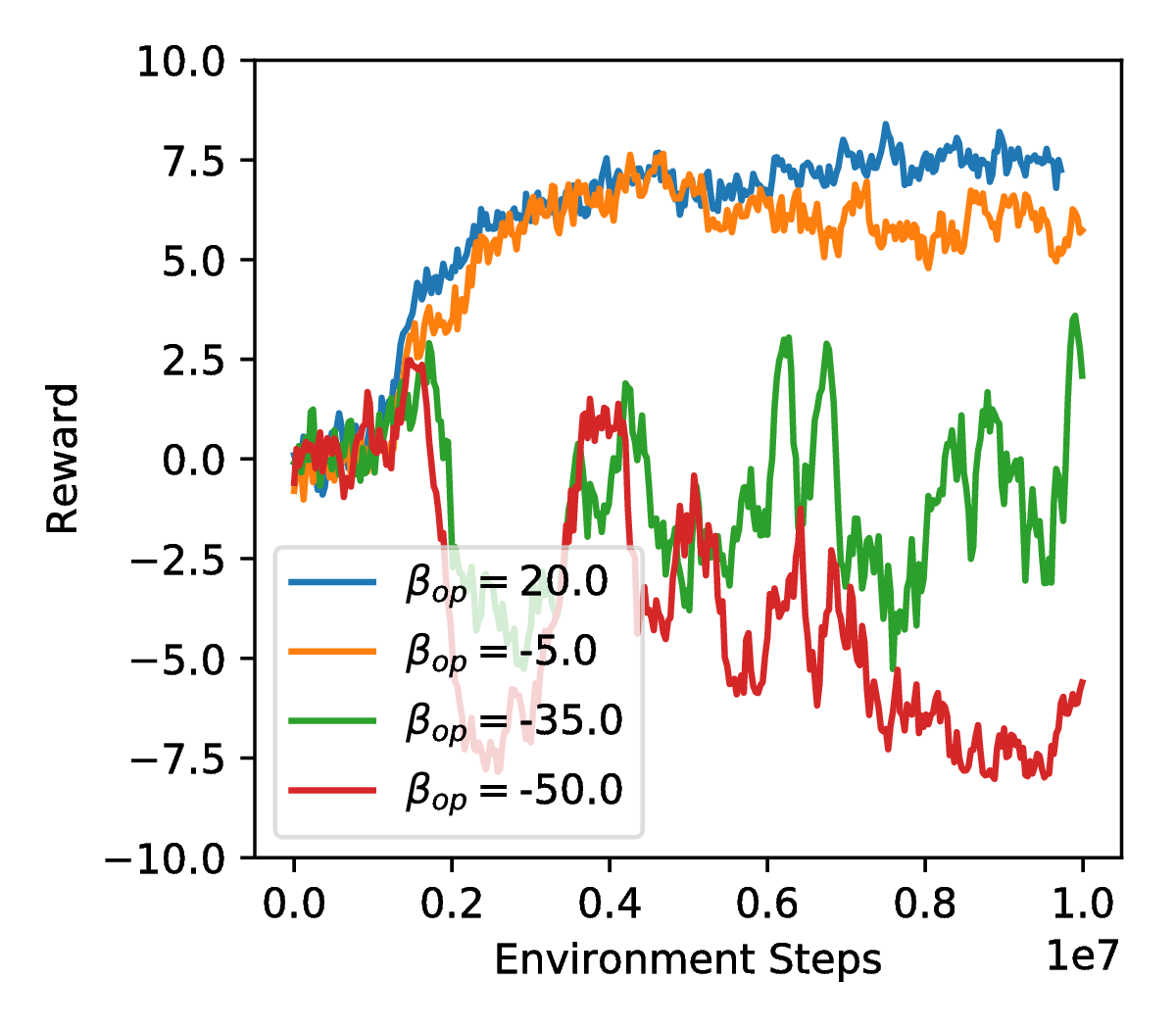}
	\caption{Results on Pong showing player's ($\beta_\text{pl} = 20$) performance depending on $\beta_\text{op}$. We see that lower values of $\beta_{\text{op}}$ yield more aggressive opponents, depicting lower reward for the player.}
	\label{fig:PongOne}
\end{figure}

\begin{figure}[h!]
    \centering
    \includegraphics[width=0.9\columnwidth]{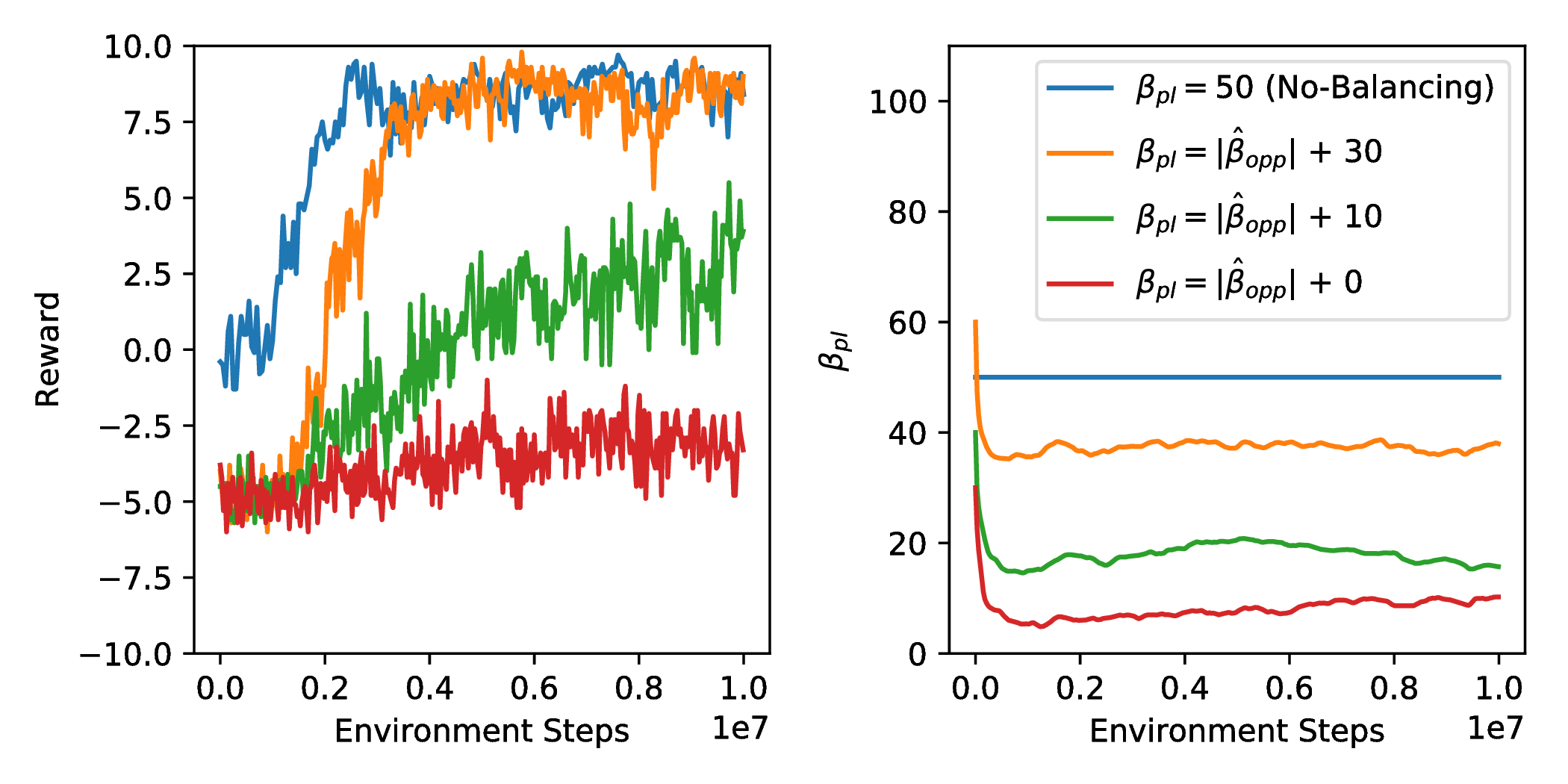}
    \caption{Player's performance depending on our game balancing scheme (see Section~\ref{Sec:ML}). We see that without balancing, the (fixed) player is much stronger than the opponent, whereas we obtain different performances depending on the balance parameter $\Delta$. On the right, we show the parameter $\beta_\text{pl}$ adapted online through the current $\hat \beta_\text{op}$ estimate. }\vspace{-0.5cm}
    \label{fig:PongTwo}
\end{figure}

\textbf{Estimating $\beta_\text{op}$ and game balancing:} Finally, we assess the performance of the maximum likelihood estimator applied to game balancing  using neural networks. We pre-trained a policy for the opponent with parameters $\beta_{\text{pl}}^\star = 50.0$ and $\beta_{\text{op}} = -20.0$, thus the player being stronger than the opponent (see blue line in Figure~\ref{fig:PongTwo}). In Figure~\ref{fig:PongTwo}, we demonstrate game balancing using Section~\ref{Sec:ML}. In particular, we are able to vary the player's performance by adapting (online) $\beta_\text{pl}$. For instance, if we set $\beta_\text{pl}$ close to $\beta_\text{op}$ we observe that the player is as strong as the opponent attaining $0$ reward, see green line.

\section{Conclusion}
We extended two-player stochastic games to agents with KL constraints. We evaluated our method theoretically and empirically in both small and high-dimensional state spaces.
The most interesting direction for future work is to scale our method to a large number of interacting agents by extending the approach in~\cite{mguni2018}.

\bibliographystyle{named}
\bibliography{bibliography}

\appendix

\section{Appendix}
\subsection{Proof Lemma 1}

\begin{proof}
The function $f(\pi_{\text{pl}}, \pi_{\text{op}}, \bm s, \mathcal V)$ is a concave function in $\pi_{\text{pl}}$  when fixing $\pi_{\text{op}}$ because the terms $\mathbb E_{\pi_{\text{pl}} \pi_{\text{op}}}[\mathcal R_{\mathcal G}]$ and $\gamma \mathbb E_{\pi_{\text{pl}} \pi_{\text{op}} \mathcal T_{\mathcal G}} \mathcal V(\bm{s}')$ are linear in $\pi_{\text{pl}}$, the last term is constant and the relative entropy term $-\frac{1}{\beta_{\text{pl}}}\text{KL}(\pi_{\text{pl}} || \rho_{\text{pl}})$ is concave for  $\beta_{\text{pl}} >0$. Similarly, $f(\pi_{\text{pl}}, \pi_{\text{op}}, \bm s , \mathcal V)$ is convex in $\pi_{\text{op}}$ when fixing $\pi_{\text{pl}}$  because  $- \frac{1}{\beta_{\text{op}}}\text{KL}(\pi_{\text{op}}||  \rho_{\text{op}})$ is convex for $\beta_{\text{op}} < 0$ and all the other term are linear or constant in $\pi_{\text{op}}$. If $f(\pi_{\text{pl}},\pi_{\text{op}}, \bm s, \mathcal F) $ is a concave-convex function then 
\begin{align*}
&\max_{\pi_\text{pl}} \min_{\pi_{\text{op}}} f(\pi_{\text{pl}},\pi_{\text{op}}, \bm s, \mathcal V) = \min_{\pi_{\text{op}}} \max_{\pi_{\text{pl}}} f(\pi_{\text{pl}},  \pi_{\text{op}}, \bm s, \mathcal V).
\end{align*}
For the remaining case it is trivial to show that 
\begin{equation}
\max_{\pi_\text{pl}} \max_{\pi_{\text{op}}} f(\pi_{\text{pl}},\pi_{\text{op}}, \bm s, \mathcal V) = \max_{\pi_{\text{op}}} \max_{\pi_{\text{pl}}} f(\pi_{\text{pl}},  \pi_{\text{op}}, \bm s, \mathcal V). \nonumber
\end{equation}
Therefore, $\mathcal B_{\text{pl}} \mathcal V(\bm s)= \mathcal B_{\text{op}} \mathcal V(\bm s)$.
\end{proof}

\subsection{Proof of Theorem 1}
We start by proving two propositions that we use later in the proof of Theorem 1.
\begin{proposition}
\begin{equation*}
	\lvert\max_x f(x) -\max_x g(x)\rvert \le \max_x \lvert f(x) - g(x) \rvert
\end{equation*}
\end{proposition}
\begin{proof}
	Given that 
	\begin{align*}
		&|x| = \lvert -x \rvert \implies \\
	  &|\max_x f(x) -\max_x g(x)|= |\max_x g(x) -\max_x f(x)|,
	\end{align*}
	 we can assume without loss of generality that  $\max_x f(x) \ge \max_x g(x)$. Let $x' = \arg \max f(x)$ then
	\begin{align*}
	|\max_x f(x) -\max_x g(x)| & \le 
	 |f(x') - g(x')| \\
	 &\le \max_{x} |f(x) - g(x)| 
	\end{align*}
\end{proof}
\begin{proposition}
\begin{equation*}
	|\min_x f(x) -\min_x g(x)| \le \max_x |f(x) - g(x)|
\end{equation*}
\end{proposition}
\begin{proof}
	Given that 
	\begin{align*}
		&|x| = \lvert -x\rvert \implies \\
		 &|\min_x f(x) -\min_x g(x)|= |\min_x g(x) -\min_x f(x)|,
	\end{align*}
	 Therefore, without loss of generality we can assume $\min_x f(x) \ge \min_x g(x)$. Let $x' = \arg \min g(x)$ then
		\begin{align*}
		|\min_x f(x) -\min_x g(x)| &\le |f(x') - g(x')|\\
		& \le \max_{x} |f(x) - g(x)| 
		\end{align*}
\end{proof}

\begin{corollary}\label{col:extremums}
	From Proposition 1 and 2 we can conclude that  $\lvert \ext_x f(x) - \ext_x g(x)\rvert \le \max_x \lvert f(x) - g(x)\rvert$ where both extremum operators are equal and either $\max$ or $\min$.
\end{corollary}

\paragraph{Proof Theorem 1}
Now we continue with the full proof of Theorem 1.
\begin{proof}
To show contraction, we start by explicitly rewriting the infinity norm as 
\begin{align*}
	& \|  \mathcal{B} \mathcal{V} - \mathcal{B} \mathcal{\bar{V}} \|_\infty = \max_{\bm{s}\in \mathcal{S}} \left| \mathcal{B}\mathcal{V}(\bm{s}) - \mathcal{B} \mathcal{\bar{V}}(\bm{s})  \right|\\
	&= \max_{\bm{s} \in \mathcal{S}} \Big| \max_{\pi_{\text{pl}}} \ext_{\pi_{\text{op}}}  f(\pi_{\text{pl}}, \pi_{\text{op}}, \bm{s}, \mathcal{V})  \\
& \hspace{10em}- \max_{\pi_{\text{pl}}} \ext_{\pi_{\text{op}}} f(\pi_{\text{pl}}, \pi_{\text{op}}, \bm{s}, \mathcal{\bar{V}}) \Big | \\
	& = \max_{\bm{s} \in \mathcal{S}} \Bigg| \max_{\pi_{\text{pl}}} \mathbb E_{\pi_{\text{pl}}} \left[ \mathcal{Q}_{\text{pl}}(\bm{s},\bm{a}^{\text{pl}}) \right] - \frac{1}{\beta_{\text{pl}}} \text{KL} (\pi_{\text{pl}}||\rho_{\text{pl}}) \\
	& \hspace{5em} -\max_{\pi_{\text{pl}}} \mathbb E_{\pi_{\text{pl}}} \left[ \mathcal{\bar{Q}}_\text{pl}(\bm{s},\bm{a}^{\text{pl}}) \right] - \frac{1}{\beta_{\text{pl}}} \text{KL}(\pi_{\text{pl}} || \rho_{\text{pl}}) \Bigg|,
\end{align*}
where from the second equality to the third we solved $\ext_{\pi_{\text{op}}}f(\pi_{\text{pl}}, \pi_{\text{op}}, \bm{s}, \mathcal{{F}})$ and,  $\mathcal{Q}_\text{pl}(\bm{s},\bm{a}^{\text{pl}})$ and $\mathcal{\bar{Q}}_\text{pl}(\bm{s},\bm{a}^{\text{pl}})$ are computed as in the equations from the main text (that depend on $\mathcal Q(s, a^{\text{pl}}, a^{\text{pl}})$ and $\mathcal V(s)$ recursively). By using the extremum operator (that can be either $\max$ or $\min$) we cover the cases where $\beta_{\text{op}}$ is either positive or negative. We continue with the proof by applying Corollary~\ref{col:extremums} to the last equality, which gives 
\begin{align*}
\|  \mathcal{B} \mathcal{V} - \mathcal{B} \mathcal{\bar{V}} \|_\infty 
	& \le \max_{\bm{s}\in \mathcal{S}} \max_{\bm{a}^{\text{pl}}\in\mathcal{A}} \left| \mathcal{Q}_\text{pl}(\bm{s},\bm{a}^{\text{pl}}) - \mathcal{\bar{Q}}_\text{pl}(\bm{s},\bm{a}^{\text{pl}}) \right|. 
\end{align*}
Note that this is valid for both cases when having negative $\beta_{\text{pl}}$ (minimization), whereas the second inequality correspond to positive $\beta_{\text{pl}}$ (maximization). Therefore our proof will cover both, positive and negative values of $\beta_{\text{pl}}$. Now, we are ready to handle the right-side of the equation making use  of Corollary~\ref{col:extremums} once again,
\begin{align*}
	& \max_{\bm{s} \in \mathcal{S}} \max_{\bm{a}^{\text{pl}}\in \mathcal{A}} \left| \mathcal{Q}_\text{pl}(\bm{s},\bm{a}^{\text{pl}}) - \mathcal{\bar{Q}}_\text{pl}(\bm{s},\bm{a}^{\text{pl}}) \right|  \\
&= \max_{\bm{s} \in \mathcal{S}} \max_{\bm{a}^{\text{pl}}\in \mathcal{A}}\Bigg| \ext_{\pi_{\text{op}}} \bigg[  \mathbb {E}_{\pi_{\text{op}}} [\mathcal{R}_{\mathcal{G}}(\bm{s},\bm{a}^{\text{pl}},\bm{a}^{\text{op}} ) \\
& \quad +  \gamma \mathbb {E}_{\mathcal{T}(\bm{s}^{\prime}|\bm{s},\bm{a}^{\text{pl}},\bm{a}^{\text{op}}) }\mathcal{V}(\bm{s}^{\prime})] - \frac{1}{\beta_{\text{op}}} {\text{KL}}(\pi_{\text{op}}||\rho_{\text{op}}) \bigg] \\
	& \quad - \ext_{\pi_{\text{op}}} \bigg[ \mathbb {E}_{\pi_{\text{op}}} [\mathcal{R}_{\mathcal{G}}(\bm{s},\bm{a}^{\text{pl}},\bm{a}^{\text{op}} )  \\
& \quad	+\gamma  \mathbb{E}_{\mathcal{T}_{\mathcal{G}}(\bm{s}^{\prime}|\bm{s},\bm{a}^{\text{pl}},\bm{a}^{\text{op}}) }\mathcal{\bar {V}}(\bm{s}^{\prime})] - \frac{1}{\beta_{\text{op}}} {\text{KL}}(\pi_{\text{op}}||\rho_{\text{op}})  \bigg] \Bigg| \\
	& \le \max_{\bm{s} \in \mathcal{S}} \max_{\bm{a}^{\text{pl}} \in\mathcal{A}} \max_{\bm{a}^{\text{op}}\in \mathcal{A}} \left| \gamma \mathbb E_{\mathcal{T}_{\mathcal{G}}}(\mathcal{V}(\bm{s}^{\prime})- \mathcal{\bar{V}}(\bm{s}^{\prime}) )\right|\\
	& \le \gamma \max_{\bm{s}^{\prime} \in \mathcal{S}}  \left| \mathcal{V}(\bm{s}^{\prime}) - \mathcal{\bar{V}}(\bm{s}^{\prime})  \right|\\
	& = \gamma \| \mathcal{V} - \mathcal{\bar{V}}\|_\infty.
\end{align*}	
\end{proof}

\subsection{Derivation Bounded Optimal Policies}
In this section we sketch the derivation of the bounded optimal policies, first, for the player and, second,  for the opponent. The player chooses its policy $\pi_{\text{pl}}$ by first doing the extremization $\ext_{\pi_{\text{op}}}$ and then its own maximization $\max_{\pi_{\text{pl}}}$.  
\begin{align*}
&\pi_{\text{pl}}^\star(\bm a^{\text{pl}}|\bm s) = \arg\max_{\pi_{\text{pl}}} \sum_{\bm{a}^{\text{pl}}} \pi_{\text{pl}}(\bm{a}^{\text{pl}}|\bm{s}) \Bigg( \\ \nonumber
&\frac{1}{\beta_{\text{op}}}  \log \sum_{\bm{a}^{\text{op}}} \rho_{\text{op}} (\bm{a}^{\text{op}}|\bm{s})  \exp\Big(\beta_{\text{op}} \big( \mathcal R_\mathcal G(\cdot) + \gamma \mathbb{E}_{\mathcal{T}_{\mathcal{G}}}[\mathcal{V}^{\star}(\bm{s}^{\prime})] \big)\Big)\\
&-\frac{1}{\beta_{\text{pl}}} \log \frac{\pi_{\text{pl}}(\bm{a}^{\text{pl}}|\bm{s})}{\rho_{\text{pl}}(\bm{a}^{\text{pl}}|\bm{s})}  
\Bigg).
\end{align*}
Solving the maximization problem $\max_{\pi_{\text{pl}}}$ by applying standard variational calculus we obtain the equation in the main manuscript. 

In contrast to the case of the player, the policy of the opponent  $\pi_{\text{op}}$ is computed by interchanging the extremization operators from $\max_{\pi_{\text{pl}}} \ext_{\pi_{\text{op}}}$ to $ \ext_{\pi_{\text{op}}}\max_{\pi_{\text{pl}}}$. Therefore, we have to solve first the inner maximization problem $\max_{\text{pl}}$ over the player's policy and then its own extremization $\ext_{\pi_{\text{op}}}$ (that is a maximization for  $\beta_{\text{op}}>0$ and a minimization for $\beta_{\text{op}}<0$). Solving first for  $\max_{\text{pl}}$ gives
\begin{align*}
&\pi_{\text{op}}^\star(\bm a^{\text{op}}|\bm s) = \arg \ext_{\pi_{\text{op}}} \sum_{\bm{a}^{\text{op}}} \pi_{\text{op}}(\bm{a}^{\text{op}}|\bm{s}) \Bigg( \\ \nonumber
&\frac{1}{\beta_{\text{pl}}}  \log \sum_{\bm{a}^{\text{pl}}} \rho_{\text{pl}} (\bm{a}^{\text{pl}}|\bm{s})\exp\Big(\beta_{\text{pl}} \big( \mathcal R_\mathcal G(\cdot) + \gamma \mathbb{E}_{\mathcal{T}_{\mathcal{G}}}\mathcal{V}^{\star}(\bm{s}^{\prime}) \big)\Big)\\
&-\frac{1}{\beta_{\text{op}}} \log \frac{\pi_{\text{op}}(\bm{a}^{\text{op}}|\bm{s})}{\rho_{\text{op}}(\bm{a}^{\text{op}}|\bm{s})}  
\Bigg).
\end{align*}
Similarly, by applying standard variational calculus we can solve for $\ext_{\pi_{\text{op}}}$ that gives the policy for the opponent written in the main manuscript.

\subsection{Maximum likelihood for estimation of $\beta_{\text{op}}$}

Consider the dataset of the form $\mathcal{D} = \left\lbrace \bm{s}_{i},\bm{a}_{i}^{\text{pl}},\bm{a}_{i}^{\text{op}} \right\rbrace_{i=1}^{m}$ with $m$ being the total number of data points. The actions of the opponent are sampled according to a fixed  distribution $\bm{a}_{i}^{\text{op}} \sim \hat \pi(\bm{a}_{i}^{\text{op}} | \bm s_i)$ unknown to the player that can be approximated with the player's model of the opponent $\pi_{\text{op}}^{\star} (\bm{a}^{\text{op}}_{i}|\bm{s}_{i})$ that is parametrized by the estimate of $\beta_{\text{op}}$. The likelihood of the data $\mathcal D$ can be written using the player's model of the opponent  as
\begin{align*}
P(\mathcal{D}|\beta_{\text{op}}) &= \prod_{i=1}^{m} \pi_{\text{op}}^{\star} (\bm{a}^{\text{op}}_{i}|\bm{s}_{i}) \\
&= \prod_{i=1}^{m} \frac{1}{Z_{\text{op}}(\bm{s}_{i})} \rho_{\text{op}}(\bm{a}_{i}^{\text{op}}|\bm{s}_{i}) \exp\Bigg(\beta_{\text{op}}\Bigg(\frac{1}{\beta_{\text{pl}}}\log \sum_{\bm{a}^{\text{pl}}} \\
&\hspace{3em}\rho_{\text{pl}}(\bm{a}^{\text{pl}}|\bm{s}_{i})\exp\Bigg(\beta_{\text{pl}}\mathcal{Q}(\bm{s},\bm{a}^{\text{pl}},\bm{a}^{\text{op}})\Bigg)\Bigg)\Bigg),
\end{align*}
where 
\begin{align*}
Z(\bm{s}_{i}) = \sum_{\bm{a}_{i}^{\text{op}}} \rho_{\text{op}}(\bm{a}_{i}^{\text{op}}|\bm{s}_{i})\exp\Bigg(\frac{\beta_{\text{op}}}{\beta_{\text{pl}}}\log \sum_{\bm{a}^{\text{pl}}}\rho_{\text{pl}}(\bm{a}^{\text{pl}}|\bm{s}_{i})  \\
\exp\Bigg(\beta_{\text{pl}}\mathcal{Q}(\bm{s},\bm{a}^{\text{pl}},\bm{a}^{\text{op}})\Bigg)
\Bigg),
\end{align*}
Then the partial derivative with respect to $\beta_{\text{op}}$ is
\begin{align*}
 	& \frac{\partial }{\partial \beta_{\text{op}}} \log P(\mathcal D | \beta_{\text{op}} ) = \sum_{i=1}^m \Bigg[ \frac{1}{\beta_{\text{op}}} \log Z(\bm{s}_i, \bm{a}_i^{\text{op}}) - \\
	&  \quad \sum_{\bm{a}_i^{\text{op}} \in \mathcal A} \frac{1}{Z(\bm{s}_i)}  \rho_{\text{op}}(\bm{a}_i^{\text{op}}|\bm{s}_i) \exp \bigg( \frac{\beta_{\text{op}}}{\beta_{\text{pl}}} \log Z(\bm{s}_i,\bm{a}_i^\text{op}) \bigg)  \times \\ 
	& \hspace{14em}\frac{1}{\beta_{\text{pl}}} \log Z(\bm{s}_i,\bm{a}_i^\text{op}) \Bigg]\\
 	& = \frac{1}{\beta_{\text{pl}}} \sum_{i=1}^m \bigg[  \log Z(\bm{s}_i,\bm{a}_i^\text{op}) \\
 	& \hspace{2cm} -  \sum_{\bm a_i^{\text{op}} \in \mathcal A} \pi^\star_{\text{op}}(\bm{a}_i^\text{op} |\bm s_i) \log Z(\bm{s}_i,\bm{a}_i^\text{op}) \bigg],
\end{align*}
where 
\begin{equation*}
	Z(\bm s_i, \bm a_i^\text{op}) = \sum_{\bm{a}^\text{pl}} \rho_\text{pl} (\bm a^{\text{pl}} |\bm s_i) \exp ( \beta_\text{pl} \mathcal Q(\bm s_i,\bm a^\text{pl},\bm a_i^{\text{op}})).
\end{equation*}
This gradient is readily applicable given that the normalizing functions can be computed exactly due to the action space being discrete. Next we give the complete description of the algorithm when learning the soft values and $\beta_{\text{op}}$ simultaneously.
\begin{algorithm}\caption{Tabular Two-Player Soft Q-Learning with $\beta_{\text{op}}$ estimation.}
\begin{algorithmic}[1]
\STATE Given $\rho_{\text{pl}}$, $\rho_{\text{op}}$, $\beta_{\text{pl}}$, $\mathcal A$, $\mathcal S$ and learning rates $\alpha$ and $\alpha_2$
\STATE $\mathcal{Q} (\bm s, \bm a^{\text{pl}}, \bm a^{\text{op}}) \gets 0$ 
\STATE $\beta_{\text{op}} \gets$ arbitrary initial $\beta_{\text{op}}$ estimate
\WHILE{not converged}
\STATE Collect transition $\big(\bm s_t, \bm a^{\text{pl}}_t, \bm a^{\text{op}}_t, \mathcal{R}_{\text{t}}, \bm{s'}_t \big) $, where $\bm{a}^{\text{pl}} \sim \pi_{\text{pl}}$, $\bm{a}^{\text{op}} \sim \pi_{\text{op}}$ and $\mathcal R_t$ is the reward at time $t$. 
\STATE $\mathcal{Q} (\bm s_t, \bm a^{\text{pl}}_t, \bm a^{\text{op}}_t) \gets \mathcal{Q} (\bm s_t, \bm a^{\text{pl}}_t, \bm a^{\text{op}}_t) + \alpha \Big( \mathcal R_{t} + \gamma \mathcal V(\bm{s'}_t) - \mathcal Q (\bm s_t, \bm a^{\text{pl}}_t, \bm a^{\text{op}}_t) \Big)$, 
\STATE $\beta_{\text{op}} \gets \beta_{\text{op}} + \alpha_2 \frac{\partial}{\partial \beta_{\text{op}}} \log P(\mathcal D | \beta_{\text{op}} )$ ($\alpha_2$: learning rate)
\ENDWHILE

\STATE {\bf return} $\mathcal{Q} (\bm s, \bm a^{\text{pl}}, \bm a^{\text{op}})$
\end{algorithmic}
\label{Algo:tabular_with_beta_estimation}
\end{algorithm}

\end{document}